\documentclass[nohyperref]{article}

\usepackage{microtype}
\usepackage{graphicx}
\usepackage{subcaption}
\usepackage{booktabs} 

\usepackage{bm}
\usepackage{hyperref}

\usepackage[accepted]{icml2023-preprint}

\usepackage{amsmath}
\usepackage{amssymb}
\usepackage{mathtools}
\usepackage{amsthm}
\usepackage{xspace}
\usepackage{makecell}
\usepackage[capitalize,noabbrev]{cleveref}

\theoremstyle{plain}
\newtheorem{theorem}{Theorem}[section]
\newtheorem{proposition}[theorem]{Proposition}

\theoremstyle{definition}
\newtheorem{definition}[theorem]{Definition}

\theoremstyle{remark}

\newcommand{\eat}[1]{}

\newcommand{\stitle}[1]{\vspace{1.6ex}\noindent{\bf #1}}

\newcommand{\ie}{\emph{i.e.,}\xspace}

\newcommand{\etc}{\emph{etc.}\xspace}

\newcommand{\bi}{\begin{itemize}}
\newcommand{\ei}{\end{itemize}}

\newcommand{\equref}[1]{{Equation~(\ref{#1})}}
\newcommand{\secref}[1]{{Section~\ref{#1}}}
\newcommand{\figref}[1]{{Figure~\ref{#1}}}

\newcommand{\tabref}[1]{{Table~\ref{#1}}}
\newcommand{\propref}[1]{{Proposition~\ref{#1}}}

\newcommand{\tabincell}[2]{
\begin{tabular}{@{}#1@{}}#2\end{tabular}
}
\newcommand{\expnumber}[2]{{#1}\mathrm{e}{#2}}

\newcommand{\normlap}{{\bm{\tilde{L}}}}
\newcommand{\model}{{LON-GNN}}
\newif \ifsourcecode
\sourcecodetrue

\usepackage[textsize=tiny]{todonotes}

\icmltitlerunning{\model: Spectral GNNs with Learnable Orthonormal Basis}

\begin{document}

\twocolumn[
\icmltitle{\model: Spectral GNNs with Learnable Orthonormal Basis}

\icmlsetsymbol{equal}{*}

\begin{icmlauthorlist}
\icmlauthor{Qian Tao}{equal,ali}
\icmlauthor{Zhen Wang}{equal,ali}
\icmlauthor{Wenyuan Yu}{ali}
\icmlauthor{Yaliang Li}{ali}
\icmlauthor{Zhewei Wei}{ruc}
\end{icmlauthorlist}

\icmlaffiliation{ali}{Alibaba Group}
\icmlaffiliation{ruc}{Renmin University of China}

\icmlcorrespondingauthor{Wenyuan Yu}{wenyuan.ywy@alibaba-inc.com}
\icmlkeywords{Graph Neural Networks, spectral GNNs, Orthonormal Polynomials, Regularization}

\vskip 0.3in
]
\printAffiliationsAndNotice{\icmlEqualContribution} 

\begin{abstract}
In recent years, a plethora of spectral graph neural networks (GNN) methods have utilized polynomial basis with learnable coefficients to achieve top-tier performances on many node-level tasks. Although various kinds of polynomial bases have been explored, each such method adopts a fixed polynomial basis which might not be the optimal choice for the given graph. Besides, we identify the so-called over-passing issue of these methods and show that it is somewhat rooted in their less-principled regularization strategy and unnormalized basis. In this paper, we make the first attempts to address these two issues. Leveraging Jacobi polynomials, we design a novel spectral GNN, \model, with \textbf{L}earnable \textbf{O}rtho\textbf{N}ormal bases and prove that regularizing coefficients becomes equivalent to regularizing the norm of learned filter function now. We conduct extensive experiments on diverse graph datasets to evaluate the fitting and generalization capability of \model, where the results imply its superiority.
\end{abstract}

\section{Introduction}
\label{sec:intro}
Various kinds of graph neural network (GNN) models have been proposed in recent years, which are often motivated and designed from either spatial or spectral domains. The expressiveness of GNNs from a spatial viewpoint is often compared with the Weisfeiler-Lehman isomorphism test of various orders~\cite{iclr19powerful}. As for spectral GNNs, recent advances that make them adaptive filters allow them to fit arbitrary graph filters~\cite{corr20briginggap}. As a result, we have witnessed many successes achieved by spectral GNNs, notably the leading performances on many node-level tasks~\cite{icml22chebyrevisit}.

Despite their consistent expressiveness, different kinds of polynomial bases are explored with various motivations, such as their difference in convergence rate~\cite{icml22jacobi}. Researchers have gained insights into the pros and cons of different polynomial bases, but there is still a lack of principles for choosing appropriate ones. Hence, existing methods are instantiated with a specific polynomial basis, and practitioners usually have to seek a suitable choice from pre-defined candidates in a trial-and-error manner.

Meanwhile, no matter which polynomial basis is adopted, recently proposed spectral GNNs are designed to decouple the transformation of node features and the filtering of graph signals~\cite{nips22evennet}, where the coefficients of the polynomial basis and the weights responsible for feature transformation are trainable parameters. In this paper, we notice that, due to the analytical form that coefficients participate in the output of such a spectral GNN, it is likely to increase the magnitude of these coefficients to reduce the training loss trivially. We call this phenomenon ``over-passing'' and provide both empirical and theoretical evidence (see Section~\ref{sec:motivations}). Existing works regularize these coefficients by penalizing their $\ell_2$-norm, with unified strength for coefficients corresponding to different terms. However, the ultimate output is a weighted sum of those terms. Their contributions in terms of the scale of magnitude differ since existing spectral GNNs adopt polynomial bases that are not orthonormal (e.g., GPRGNN~\cite{iclr21gprgnn}) or even not orthogonal (e.g., BernNet~\cite{nips21bernnet}). Thus, penalizing all the coefficients with unified strength might be problematic.

Based on the above discussions, it is natural to raise two questions: (1) Is it possible to not only learn the coefficients but also, in the meantime, learn the suitable polynomial basis? (2) How can we properly regularize a spectral GNN by simply penalizing the $\ell_2$-norm of all the coefficients?

In this paper, we attempt to answer these questions by exploiting some properties of the Jacobi polynomials $P_{n}^{a,b}(\cdot)$, a class of orthogonal polynomials. Specifically, we propose a simple yet effective spectral GNN named \model, which is learned by jointly optimizing the trainable parameters of vanilla spectral GNNs as usual, as well as the parameters of Jacobi polynomials (i.e., $a$ and $b$) so that the choice of polynomial basis is also tuned during the training course. Besides, we normalize all the terms of the applied polynomial basis by the norm of each term. Then we theoretically show that regularizing all the coefficients with a unified strength would be equivalent to regularizing the norm of filter function, as the applied polynomial basis has become orthonormal. Noticeably, the norms used for normalization can be calculated analytically and is differentiable regarding the parameters of Jacobi polynomials. We conduct extensive comparisons, including fitting ground-truth filters and node classification tasks, to evaluate \model, where experimental results confirm the effectiveness of \model. Meanwhile, more in-depth empirical studies confirm the advantages of learnable and orthonormal polynomial basis.
\section{Preliminaries}
\label{sec:pre}
Let $\mathcal{G}=(\mathcal{V},\mathcal{E})$ denote an undirected graph, where $\mathcal{V}$ is a set of nodes with cardinality $n=|\mathcal{V}|$, and $\mathcal{E}$ is a set of edges.
The graph may also have features for each node, which are represented as a matrix $\bm{X}\in \Re^{n\times r}$, where $r$ is the dimension of the features. 
The adjacency matrix of the graph is represented by $\bm{A}$, such that $\bm{A}_{ij}=1$ if there is an edge between nodes $v_i$ and $v_j$, and $\bm{A}_{ij}=0$ otherwise. 
The degree matrix $\bm{D}$ is a diagonal matrix, where $\bm{D}_{ii}$ is the degree of node $v_i$, and $\bm{D}_{ij}=0$ for all $i\neq j$.
Graph Laplacian is then defined as $\bm{L}=\bm{D}-\bm{A}$.
We follow the convention to focus on its normalized counterpart $\normlap=\bm{D}^{-\frac{1}{2}}\bm{L}\bm{D}^{-\frac{1}{2}}=\bm{I}-\bm{D}^{-\frac{1}{2}}\bm{A}\bm{D}^{-\frac{1}{2}}$.
Considering node classification tasks, there is often a label matrix $\bm{Y}\in\Re^{n\times c}$, where $c$ is the number of classes, and each $i$-th row is a one-hot vector indicating which class $v_i$ belongs to.

Regarding $\bm{X}$ as graph signals, the graph filtering operation is defined as $\bm{Z}=\sigma(\bm{U}g(\Lambda)\bm{U}^{\text{T}})$, where $\Lambda$ denotes the diagonal eigenvalue matrix of $\normlap$, and $\sigma(\cdot)$ represents a normalization function. To avoid the computationally expensive eigendecomposition, most existing spectral GNN methods employ polynomials to approximate graph filter $g(\Lambda)$~\cite{iclr21gprgnn,nips21bernnet,icml22jacobi,icml22chebyrevisit,nips22evennet}:
\[
\bm{U}g(\Lambda)\bm{U}^{\text{T}}\bm{X}\approx\bm{U}(\sum_{k=0}^{K}\alpha_{i}\Lambda^{k})\bm{U}^{\text{T}}\bm{X}=\sum_{k=0}^{K}\alpha_{k}\normlap^{k}\bm{X},
\]
where $\{\alpha_k\}$ are coefficients.
These methods treat $\{\alpha_k\}$ as trainable parameters and thus can restate the graph filter with the propagation matrix $\bm{P}=\bm{I}-\normlap=\bm{D}^{\frac{1}{2}}\bm{A}\bm{D}^{-\frac{1}{2}}$ as:
\begin{equation}
g(\Lambda)=g(\normlap)=\sum_{k=0}^{K}\alpha_{k}\normlap^{k}=\sum_{k=0}^{K}\alpha_{k}P_{k}(\bm{P}),
\label{eq:polyapproxgf}
\end{equation}
where $\{P_{k}(\cdot)\}$ is a polynomial basis. Research works in this line have explored various kinds of polynomial bases.

Then, for the purpose of node classification, the output of a spectral GNN has the following form:
\begin{equation}
\bm{\hat{Y}}=f_{\alpha,\theta}(\bm{X},\normlap)=g_{\alpha}(\normlap)t_{\theta}(\bm{X}),
\end{equation}
where $g_\alpha(\normlap)$ is implemented based on the right-hand side of Eq.~\eqref{eq:polyapproxgf}, and $t(\cdot)$ is a linear or an MLP model that takes $\bm{X}$ as input and outputs the embeddings of shape $n\times c$.
Hence, the trainable parameters of such a spectral GNN include both $\{\alpha_k\}$ and the model parameters for transformation $\theta$.
They are learned by minimizing a loss function $\mathcal{L}(\cdot,\cdot)$, e.g., Cross-entropy loss, defined over $\bm{\hat{Y}}$ and $\bm{Y}$.
\vspace{-0.05in}
\section{Motivations}
\label{sec:motivations}
Existing research works on spectral GNN mainly focus on the study of expressiveness~\cite{iclr21gprgnn,nips21bernnet} and optimization~\cite{icml21optim-gnn,icml22jacobi}, while how to regularize the GNN model has not been systematically investigated. They either penalize the $\ell_2$-norm of polynomial coefficients with a unified strength or ignore regularization. Therefore, we present our observations and analysis of such a regularization, which motivate us to propose \model.

\subsection{Learning Filters without Regularization: The Over-Passing Issue}
\label{subsec:motivation-overpass}
Generally, a graph filter is defined as a function that maps from $[0, 2]$ to $[0, 1]$, where the eigenvalues of $\normlap$ lie in that domain.
However, we observe that, without regularization, the output values of a learned filters are likely to dramatically increase, leading to the so-called over-passing issue.

\begin{figure}[t!]
    \centering
    \begin{subfigure}[b]{0.22\textwidth}
           \centering
           \includegraphics[width=\textwidth,trim=4 6 40 10,clip]{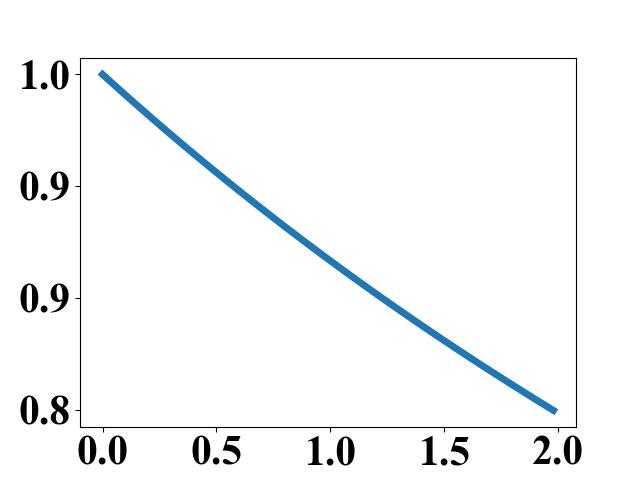}
            \caption{Filter at epoch=0.}
            \label{subfig:filter-eph40}
    \end{subfigure}
    \hspace{4pt}
    \begin{subfigure}[b]{0.22\textwidth}
            \centering
            \includegraphics[width=\textwidth,trim=4 6 40 10,clip]{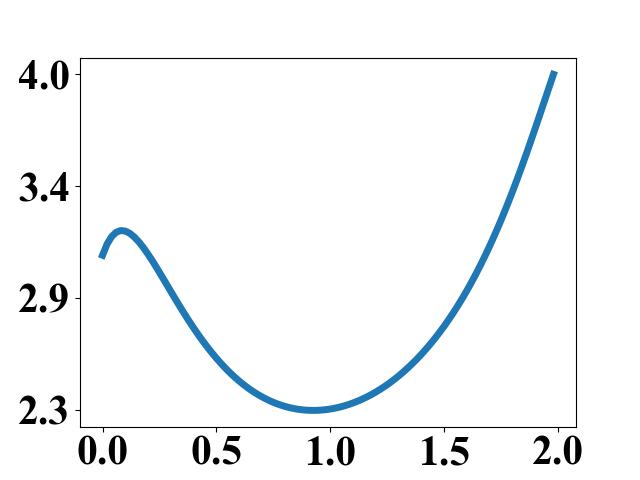}
            \caption{Filter at epoch=80.}
            \label{subfig:filter-eph80}
    \end{subfigure}
    \hfill
    \begin{subfigure}[b]{0.22\textwidth}
            \centering
            \includegraphics[width=\textwidth,trim=4 6 40 10,clip]{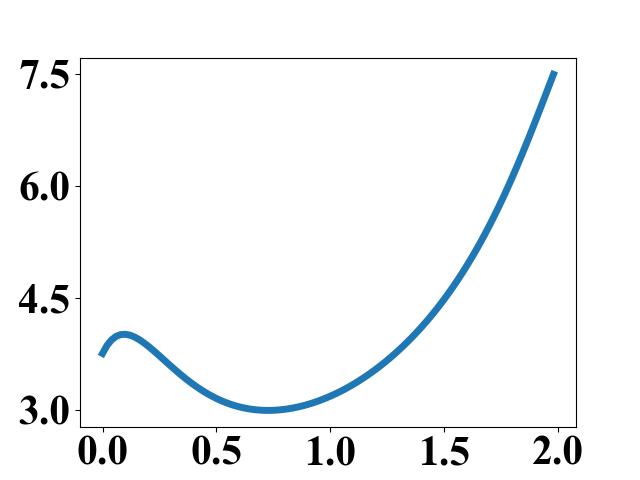}
            \caption{Filter at epoch=120.}
            \label{subfig:filter-eph120}
    \end{subfigure}
    \hspace{4pt}
    \begin{subfigure}[b]{0.22\textwidth}
        \centering
        \includegraphics[width=\textwidth,trim=4 6 40 10,clip]{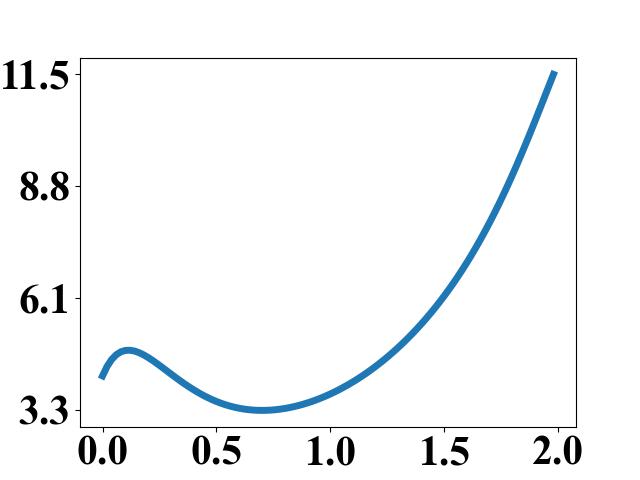}
        \caption{Filter at epoch=160.}
        \label{subfig:filter-eph160}
    \end{subfigure}
    \vspace{-0.1in}
     \caption{Learned graph filters at different epochs.}
    \label{fig:filter-ephs}
    \vspace{-0.2in}
\end{figure}

\stitle{The Over-Passing Issue.}
We present the filters learned by GPRGNN~\cite{iclr21gprgnn} on the Cornell dataset at different epochs in \figref{fig:filter-ephs}, from which we attain the following intuitive observations:
i) at the late stage of the training process (from epoch $=120$ to $160$), the \emph{waveform}~\cite{smith1997guidetodsp} (\ie the shape) of the filter remains stable but its \emph{amplitude} (\ie the range of the output) continues to increase; ii) even at the early stage when the model is learning the filter (from epoch $=0$ to $80$), the amplitude of the filter still increases; iii) starting from the $80$-th epoch, the output values lie in an invalid range, namely, become larger than one for the whole spectrums.
These observations indicate that the amplitude of the learned graph filter keeps increasing and does not contribute to the learning of desired filters.

We speculate that such an over-passing issue is rooted in an updating direction in the parameter space where trivially increasing the absolute value of polynomial coefficients can reduce the training loss, but the learned filter is not improved.
Then we formalize this relationship as follows:

\begin{proposition}
    \label{prop:over-passing}
    Assume a spectral GNN $\bm{\hat{Y}}=g_{\alpha}(\normlap)t_{\theta}(\bm{X})$ with cross entropy loss $\mathcal{L}$.
    Suppose currently the model has loss $l=\mathcal{L}(\bm{\hat{Y}}, \bm{Y})$ where $\bm{Y}$ is the ground truth of the task.
    There exists a spectral GNN $ \bm{\hat{Y}'}=g_{\alpha'}(\normlap)t_{\theta'}(\bm{X})$ satisfying:
    \begin{itemize}
        \item $\theta'$ equals $\theta$;
        \item The output $\bm{\hat{Y}'}$ results in the same prediction as $\bm{\hat{Y}}$;
        \item The loss for nodes that are predicted correctly decreases.
    \end{itemize}
\end{proposition}

\begin{proof}
    Letting $\theta'=\theta$, we have $t_{\theta}(\bm{X})=t_{\theta'}(\bm{X})$.
    Assume $\alpha_i$ the coefficients and $\{P_0(\lambda),P_1(\lambda),...,P_K(\lambda)\}$ the polynomial basis of $g_\alpha(\cdot)$, \ie $g_{\alpha}(\lambda)=\sum_{i=0}^K\alpha_iP_i(\lambda)$.
    We enlarge the coefficients with the same factor $q>1$ and choose the enlarged values as the coefficients $\{\alpha_{i}'\}$, \ie $\alpha_i'=q\alpha_{i},i=0,\ldots,K$ and $g_{\alpha'}(\lambda)=\sum_{i=0}^K\alpha_i'P_i(\lambda)$.
    As a result, each element of $g_{\alpha'}(\normlap)$ will be enlarged by $q$ and $\bm{\hat{Y}'}$ will result in the same prediction as $\bm{\hat{Y}}$.
    
    Then we analyze the variation of the loss of the nodes that are predicted correctly.
    Let $\mathcal{V}^{+}$ be the nodes that are predicted correctly by $\bm{\hat{Y}}$ and $\bm{\hat{Y}'}$.
    The loss of $\bm{\hat{Y}}$ on $\mathcal{V}^{+}$ is
    \begin{align}
        &l=\sum_{v\in \mathcal{V}^{+}} \mathcal{L}(\bm{\hat{Y}}_{v},\bm{Y}_{v}) \notag \\
        &=\sum_{v\in \mathcal{V}^{+}}-\log \frac{e^{\bm{\hat{Y}}_{v,\bm{Y}_{v}}}}{\sum_{j=1}^c e^{\bm{\hat{Y}}_{v,j}}} \notag \\
        &=\sum_{v\in \mathcal{V}^{+}}\log (\sum_{j=1}^c e^{\bm{\hat{Y}}_{v,j}-\bm{\hat{Y}}_{v,\bm{Y}_{v}}}),
    \end{align}
    where $\bm{\hat{Y}}_{i,j}$ and $\bm{Y}_i$ are the $(i,j)$-th element of $\bm{\hat{Y}}$ and the $i$-th row of $\bm{Y}$, respectively.
    For the output $\bm{\hat{Y}'}$, we have $\bm{\hat{Y}'}=q\bm{\hat{Y}}$.
    Similarly, the loss of $\bm{\hat{Y}}$ on $\mathcal{V}^{+}$ becomes
    \begin{align}
        l'&=\sum_{v\in \mathcal{V}^{+}}\log (\sum_{j=1}^c e^{\bm{\hat{Y}'}_{v,j}-\bm{\hat{Y}'}_{v,\bm{Y}_{v}}}) \notag \\
        &=\sum_{v\in \mathcal{V}^{+}}\log (\sum_{j=1}^c e^{q(\bm{\hat{Y}}_{v,j}-\bm{\hat{Y}}_{v,\bm{Y}_{v}})}).
    \end{align}
    Since each node in $\mathcal{V}^{+}$ is predicted correctly, we have $\bm{\hat{Y}}_{v,j}-\bm{\hat{Y}}_{v,\bm{Y}_{j}}\leq 0$ for any $v\in \mathcal{V}^{+}$ and $j=0,1,...,c$.
    Also noticing that $q>1$, we have $l'<l$.
\end{proof}

Note that the conclusion of \propref{prop:over-passing} could be applied to the practical spectral GNNs with nonlinearity and multiple layers like \cite{nips21bernnet,icml22jacobi}. 
It indicates that when a considerable number of nodes are predicted correctly, the model may be updated towards the direction that the accuracy has slight change but the absolute values of polynomial coefficients grow continuously.
Besides, since the coefficients multiplied by a shared factor do not change the waveform of the learned filter, the over-passing issue makes most of the gradient occupied for increasing polynomial coefficients. Such a learning process neither contributes to learning the waveform of ground-truth filters, nor produces improvement on the performance of the model, which we name the over-passing issue.

\subsection{Learning Filters with $\ell_2$-Norm Penalization: An Unequal Regularization}
\label{subsec:motivation-l2norm}
{The over-passing issue indicates that the \emph{output} of the filter functions may increase unlimitedly while training the spectral GNNs.}
Strictly requiring the range of graph filter functions within $[0,1]$ can be nontrivial for spectral GNNs.
A natural way is to relax the requirement and penalize the norm of the filter function, that is to say, limit the output values of filters via regularization.
Some works regularize the spectral GNN model by penalizing the $\ell_2$-norm of the polynomial coefficients with a unified strength.
Then the model is often learned by solving the following optimization problem:
\begin{align}
\label{equ:regu-l2norm}
\min_{\alpha,\theta}\mathcal{L}(f_{\alpha,\theta}(\bm{X},\normlap),\bm{Y})+\sum_{i=0}^K \alpha_i^2.
\end{align}
To keep brevity, here we omit the $\ell_2$-norm for regularizing other trainable parameters, i.e., $\theta$.

This $\ell_2$-norm penalty can alleviate the over-passing issue to some extent.
However, as we will show, regularizing a graph filter in such a way cannot equally regularize its output values.
Inspired by the Rooted-Mean-Square amplitude (RMS amplitude)~
\cite{smith1997guidetodsp} of the filters, we define the norm of a polynomial based on the inner product of two functions, which could be regarded as the average output value of a graph filtering function in an interval.
\begin{definition}[Norm]
    \label{def:norm}
    Given a polynomial $p(x)$, the norm of $p(x)$ in interval $[l,u]$ is 
    $$||p(x)||=\sqrt{\int_l^u p^2(x)w(x)\mathrm{d}x}\quad,$$
    where $w(x)$ is the weight function that determines the preference on values in $[l,u]$.
\end{definition}

For a spectral GNN, its polynomial function $g(\cdot)$ is applied to the propagation matrix $\bm{P}$ (see \equref{eq:polyapproxgf}), and it is well known that the eigenvalues of $\bm{P}$ are in the range $[-1, 1]$.
Thus, we use $g(\lambda)$, $g(\bm{P})$, and $g(\normlap)$ exchangeably, where $\lambda\in[-1,1]$ denotes the eigenvalues of $\bm{P}$.
Meanwhile, \textit{w.l.o.g.}, we consider $[l, u] = [-1, 1]$ as the integral interval in the remaining of this paper.

To regularize the output values of the filtering function, we can add a penalty term in the loss function:
\begin{align}
\label{equ:regu-filter-raw}
    \min_{\alpha,\theta} \mathcal{L}(f_{\alpha,\theta}(\bm{X},\normlap),\bm{Y})+||g_{\alpha}(\lambda)||^{2},
\end{align}
where $g_{\alpha}(\lambda)=\sum_{i=0}^K \alpha_i P_i(\lambda)$ with trainable coefficients $\alpha_0,...,\alpha_K$ and polynomial basis $P_0(\lambda),P_1(\lambda),...,P_K(\lambda)$.

The integral form of the regularization term $||g_{\alpha}(\lambda)||^2$ makes it difficult to calculate the gradient of the coefficients.
\eat{
\begin{align}
\label{equ:regu-raw-to-orth}
||g_{\alpha}(x)||^2&=\int_{-1}^1(\sum_{i=0}^K\alpha_iP_i(\lambda))^2w(\lambda)\mathrm{d}\lambda \\
&=\sum_{i=0}^K\int_{-1}^1\alpha_i^2P_i^2(\lambda)w(\lambda)\mathrm{d}\lambda \notag \\
\text{ }&+2\sum_{i<j}\int_{-1}^1\alpha_i\alpha_jP_i(\lambda)P_j(\lambda)w(\lambda)\mathrm{d}\lambda \quad.\notag
\end{align}
}
However, if we utilize the orthogonal polynomials that share the same weight function as the definition of norm, \ie the polynomials satisfying
$$\int_{-1}^1 P_i(\lambda)P_j(\lambda)w(\lambda)\mathrm{d}\lambda=0 \text{ } (i\neq j),$$
according to the generalized Pythagorean theorem in inner product spaces~\cite{van2005classification}, the regularization term can be simplified as
\begin{align}\label{equ:func-norm-orth}
||g_{\alpha}(\lambda)||^2=\sum_{i=0}^K\int_{-1}^1\alpha_i^2P_i^2(\lambda)w(\lambda)\mathrm{d}\lambda=\sum_{i=0}^K\alpha_i^2||P_{i}(\lambda)||^2.
\end{align}
As a result, the regularization on filtering functions (\ie \equref{equ:regu-filter-raw}) becomes 
\begin{align}\label{equ:regularization}
    \min_{\alpha,\theta}\mathcal{L}(f_{\alpha,\theta}(\bm{X},\normlap),\bm{Y})+\sum_{i=0}^K \alpha_i^2||P_i(\lambda)||^2.
\end{align}

\vspace{-6pt}
\stitle{Unequality of $\ell_2$-Norm.}
Some recent spectral GNN works~\cite{icml22chebyrevisit,icml22jacobi} have adopted orthogonal polynomials and regularized the polynomial coefficients in the way of \equref{equ:regu-l2norm}. Although that regularization term has a similar form to penalizing the norm of the filtering function (\ie ~\equref{equ:regularization}), regularizing just the trainable coefficients cannot equally regularize the output values of the filtering function due to the unnormalized polynomial terms, and thus the over-passing issue is solved in an improper way.
In other words, they are not equivalent, as illustrated in \figref{fig:regular-filter-values}. Specifically, we consider a basis of Jacobi polynomials up to degree 2, where the norm of each term (i.e., $\|P_{i}(\lambda)\|,i=0,1,2$) is calculated in advance. We choose the polynomial coefficients $\alpha_i,i=0,1,2$ located on a sphere with a radius of 5, and thus their $\ell_2$-norm $\sum_{i=0}^2 \alpha_i^2$ always has a value of 25. Then we use the color to represent how large the regularization term in \equref{equ:regularization} is. As can be seen, although the regularization on coefficients keeps unchanged on the sphere, the regularization term of filter norm, i.e., $\sum_{i=0}^2 \alpha_i^2 ||P_i(\lambda)||^2$, has different values ranging from $20$ to $40$.

\begin{figure}[t!]
    \centering
    \includegraphics[width=0.25\textwidth,trim=80 40 40 55,clip]{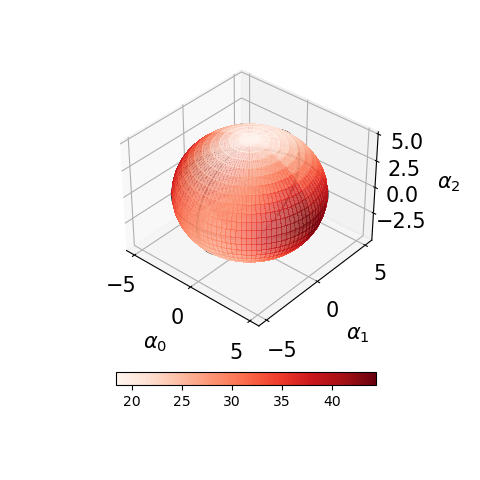}
    \caption{Color map of the regularization for filter values.}
    \label{fig:regular-filter-values}
    \vskip -0.2in
\end{figure}
\vspace{-2pt}
\section{{\model}: Spectral GNNs with Learnable Orthonormal Basis}
In this section, we propose 
{\model} (GNN with \underline{L}earnable {\underline{O}rtho{\underline{N}ormal basis}), which not only utilize the orthonormal polynomials as the basis of the spectral GNNs, but also learn the proper basis via trainable polynomials.

\subsection{From Orthogonal to Orthonormal Basis}
\label{subsec:lon-orthnorm}
The discussion in \secref{sec:motivations} inspires us to design the spectral graph neural networks based on the \textit{orthonormal polynomials}.
When an orthonormal basis $\{P^*_0(\lambda),P^*_1(\lambda),...,P^*_K(\lambda)\}$ is utilized in spectral GNNs, the regularization on coefficients is equivalent to the regularization on filters.
As a result, an $\ell_2$-norm for coefficients is enough for regularizing the filters.
\begin{theorem}
    For a spectral GNN $f_{\alpha,\theta}(\bm{X},\normlap)=g^*_\alpha(\normlap)t_\theta(\bm{X})$ with orthonormal polynomials $\{P_0^*(\lambda),P_1^*(\lambda),...,P_K^*(\lambda)\}$, a regularization on the coefficients of the polynomials is equivalent to the regularization on the filters.
\end{theorem}
\vspace{-8pt}
\begin{proof}
    Orthonormal polynomials satisfy for any $i,j\in\{0,\ldots,K\},$
    \begin{align}
        &\int_{-1}^1 (P_i^*(\lambda))^2w(\lambda)\mathrm{d}\lambda=1 \notag \\
        &\int_{-1}^1 P_i^*(\lambda)P_j^*(\lambda)w(\lambda)\mathrm{d}\lambda=0 \text{ } (i\neq j). 
    \end{align}
    A similar reduction following \equref{equ:func-norm-orth} indicates that 
    \vspace{-4pt}
    $$||g^*_\alpha(\lambda)||^2=\sum_{i=0}^K \alpha_i^2||P_i^*(x)||^2=\sum_{i=0}^K\alpha_i^{2}$$
    which is exactly the square of $\ell_2$-norm of coefficients.
\end{proof}
\vspace{-16pt}
Finally, the proposed model can be represented by 
\begin{align}\label{equ:ton-prop}
\bm{\hat{Y}}=f_{\alpha,\theta}(\bm{X},\normlap)=g^*_{\alpha}(\normlap)t_{\theta}(\bm{X}) 
\end{align}
with regularization
\vspace{-10pt}
\begin{align}
\label{equ:ton-wd}
\min_{\alpha,\theta}\mathcal{L}(f_{\alpha,\theta}(\bm{X},\normlap),\bm{Y})+\sum_{i=0}^K \alpha_i^2.
\end{align}

\begin{table}[t]
    \caption{Common Orthogonal Polynomials and Norms.}
    \label{tab:orth-poly}
    \vspace{-6pt}
    \begin{center}
    \begin{tabular}{p{1.2cm}cc}
    \toprule
    Poly & $w(\lambda)$ & $||P_i(\lambda)||^2$  \\
    \midrule
    Chebyshev    & $1/\sqrt{1-\lambda^2}$ &$
        \begin{cases}
             \pi \text{ if }i=0 \\
             \frac{\pi}{2} \text{otherwise} 
        \end{cases}$\\
    Legendre & 1 & 1/(2i+1)\\
    Jacobi    & $(1-\lambda)^a(1+\lambda)^b$ & $\frac{2^{a+b+1}\Gamma(i+a+1)\Gamma(i+b+1)}{(2i+a+b+1)\Gamma(i+a+b+1)i!}$ \\
    \bottomrule
    \end{tabular}
    \end{center}
    \vskip -0.3in
    \end{table}

\stitle{Computation of Orthonormal Polynomials.}
The orthonormal polynomials can be generated from orthogonal polynomials by dividing each basis to its norm.
The weight functions and norm of common orthogonal polynomials can be found in \tabref{tab:orth-poly}.

\begin{theorem}[\cite{book2011orth}]
    Given a set of orthogonal polynomials $\{P_0(\lambda),P_1(\lambda),...,P_K(\lambda)\}$, the set of polynomials $\{P_0^*(\lambda),P_1^*(\lambda),...,P_K^*(\lambda)\}$ is orthonormal with
    $$P_i^*(\lambda)=\frac{P_i(\lambda)}{||P_i(\lambda)||}.$$
\end{theorem}

In practice, we choose the orthonormal Jacobi polynomials as basis, because:
i) a large number of orthogonal polynomials, like Chebyshev and Legendre polynomials in \tabref{tab:orth-poly} and the Zernike polynomials~\cite{lakshminarayanan2011zernike}, can be regarded as a special case of Jacobi polynomial by choosing corresponding values of $a$ and $b$; ii) as discussed in \secref{subsec:motivation-l2norm}, we could use the parameters $a$ and $b$ of Jacobi polynomials to define the preference to the signals with different values.
The recurrence relation of orthonormal Jacobi polynomials can be defined as:
\begin{align}
    \label{equ:ton-poly}
    &P_0(\normlap)=1;\ P_1(\normlap)={0.5a-0.5b}+(0.5a+0.5b+1)\normlap \notag \\
    &P_i(\normlap)=(2i+a+b-1) \notag \\
    &\cdot\frac{(2i+a+b)(2i+a+b-2)\normlap+a^2-b^2}{2i(i+a+b)(2i+a+b-2)}P_{i-1}(\normlap) \notag \\
    &-\frac{(i+a-1)(i+b-1)(2i+a+b)}{ i(i+a+b)(2i+a+b-2)}P_{i-2}(\normlap) \notag \\
    &||P_i(\lambda)||=\sqrt{\frac{2^{a+b+1}\Gamma(i+a+1)\Gamma(i+b+1)}{(2i+a+b+1)\Gamma(i+a+b+1)i!}}\notag \\
    &P_i^*(\normlap)=\frac{P_i(\normlap)}{||P_i(\lambda)||}. 
\end{align}

\equref{equ:ton-prop}, \equref{equ:ton-wd}, and \equref{equ:ton-poly} together compose the spectral GNN based on orthonormal basis.

\subsection{Learnable Orthonormal Basis}
\label{subsec:lon-learnable}
As discussed in \secref{subsec:motivation-overpass}, the weight function of the norm of a polynomial in fact represents the preference to the values distributed on $[-1,1]$.
Consider the Jacobi polynomial, whose weight function is $(1-\lambda)^a(1+\lambda)^b$.
For example, if $a=b=0$ such that the polynomial degenerates to the Legendre polynomial, the weight function becomes a constant $1$, which means the outputs of the polynomial in $[-1,1]$ are equally treated.
However, when $a>0$ and $b=0$, the weight function will be $(1-\lambda)^a$ and monotonically decrease.
In such a case, the polynomial represents a preference for smaller values close to $-1$.
    
The above discussion inspires us to learn the orthonormal basis by making the parameters $a,b$ in Jacobi polynomials trainable.
In each epoch, the model forward propagates based on \equref{equ:ton-prop} and \equref{equ:ton-poly}, and updates $a,b$ and other parameters.
This indicates that in every epoch, {\model} formalizes a different orthonormal Jacobi polynomial based on different $a$ and $b$.

\stitle{Differentiability.}
The derivation of the Gamma function in the norm of Jacobi polynomials can be obtained by the Polygamma function:
$$\Gamma'(\lambda)=\Gamma(\lambda)\psi^{(0)}(\lambda).$$
Also notice that $a$ and $b$ in the recurrence relation (\ie~ \equref{equ:ton-poly}) are differentiable, it is feasible to learn a proper $a$ and $b$ while training in {\model}.

    \begin{table*}[th!]
    \caption{Statistics of Real-World Datasets.}
    \label{tab:exp-sta}

    \begin{center}
    \begin{tabular}{ccccccccccc}

    \toprule
      & Cora & Citeseer & Pubmed & Computers & Photo & Chameleon & Squirrel & Actor & Texas & Cornell  \\
    \midrule
    Nodes & 2708 & 3327 & 19717 & 13752 & 7650 & 2277 & 5201 & 7600 & 183 & 183 \\
    Edges & 5278 & 4552 & 44324 & 245861 & 119081 & 31371 & 198353 & 26659 & 279 & 277 \\
    Features & 1433 & 3703 & 500 & 767 & 745 & 2325 & 2089 & 932 & 1703 & 1703 \\
    Classes & 7 & 6 & 5 & 10 & 8 & 5 & 5 & 5 & 5 & 5 \\
    homophily & 0.825 & 0.718 & 0.792 & 0.802 & 0.849 & 0.247 & 0.217 & 0.215 & 0.057 & 0.301 \\
    \bottomrule 
    
    \end{tabular}
    \end{center}
\end{table*}

\begin{table}[t!]
    \caption{Average Loss in Learning Filters Experiments.}
    \label{tab:exp-res-fit}
    \vspace{4pt}
    \begin{center}
    \begin{tabular}{p{1.2cm}p{0.9cm}p{0.9cm}p{0.9cm}p{0.9cm}p{0.9cm}}

    \toprule
     & LOW & HIGH & BAND & REJECT & COMB  \\
    \midrule
    GPRGNN    & $0.4169$ & $0.0943$ & $3.5121$ & $3.7917$ & $4.6549$ \\
    ARMA    & $1.8478$ & $1.8632$ & $7.6922$ & $8.2732$ & $15.121$ \\
    ChebNet & $0.8220$ & $0.7867$ & $2.2722$ & $2.5296$ & $4.0735$ \\
    BernNet & $0.0314$ & $0.0113$ & $0.0411$ & $0.9313$ & $0.9982$ \\
    JacobiConv & $\bm{0.0003}$ & $0.0011$ & $0.0213$ & $\bm{0.0156}$ & $0.2933$ \\
    \midrule
    {TON} & $\bm{0.0003}$ & $\bm{0.0003}$ & $\bm{0.0156}$ & $\bm{0.0156}$ & $\bm{0.2870}$ \\
    \bottomrule

    \end{tabular}
    \end{center}
    \vskip -0.1in
    \end{table}
    
\subsection{Put Them Together}
\label{subsec:lon-together}

By the combination of the discussion in \secref{subsec:lon-orthnorm} and \secref{subsec:lon-learnable}, we propose {\model}.
{\model} employs the orthonormal Jacobi polynomial basis to achieve an equal regularization on filter values, and enables learnable parameters in the Jacobi polynomials such that more proper $a$ and $b$ can be achieved.

\stitle{Implementation.}
In the implementation, {\model} employs an MLP as the embedding of features following \cite{iclr21gprgnn,nips21bernnet}.
Following \cite{icml22jacobi}, we also implement multiple filters to flexibly adapt to multiple feature channels.

\begin{table*}[t!]
\centering
    \caption{Experimental Results on Real-World Datasets.}
    \label{tab:exp-real-world-res}

    \begin{tabular}{ccccccccccc}

    \toprule
     Datasets & GCN & APPNP & ChebNet & GPRGNN & BernNet & JacobiConv & {\model}  \\
    \midrule
    Cora & $87.14_{\pm 1.01}$ & $88.14_{\pm 0.73}$ & $86.67_{\pm 0.82}$ & $88.57_{\pm 0.69}$ & $88.52_{\pm 0.95}$ & $88.98_{\pm 0.46}$ & $\bm{89.44_{\pm 1.12}}$ \\
    Citeseer & $79.86_{\pm 0.67}$ & $80.47_{\pm 0.74}$ & $79.11_{\pm 0.75}$ & $80.12_{\pm 0.83}$ & $80.09_{\pm 0.79}$ & $80.78_{\pm 0.79}$ & $\bm{81.41_{\pm 1.15}}$  \\
    Pubmed & $86.74_{\pm 0.27}$ & $88.12_{\pm 0.31}$ & $87.95_{\pm 0.28}$ & $88.46_{\pm 0.33}$ & ${88.48_{\pm 0.41}}$ & ${89.62_{\pm 0.41}}$ & $\bm{90.98_{\pm 0.64}}$ \\
    Computers & $83.32_{\pm 0.33}$ & $85.32_{\pm 0.37}$ & $87.54_{\pm 0.43}$ & $86.85_{\pm 0.25}$ & $87.64_{\pm 0.44}$ & $90.39_{\pm 0.29}$ & $\bm{90.93_{\pm 0.82}}$ \\
    Photo & $88.26_{\pm 0.73}$ & $88.51_{\pm 0.31}$ & $93.77_{\pm 0.32}$ & $93.85_{\pm 0.28}$ & $93.63_{\pm 0.35}$ & $\bm{95.43_{\pm 0.23}}$ & $94.66_{\pm 0.52}$ \\
    Chameleon & $59.61_{\pm 2.21}$ & $51.84_{\pm 1.82}$ & $59.28_{\pm 1.25}$ & $67.28_{\pm 1.09}$ & $68.29_{\pm 1.58}$ & $\bm{74.20_{\pm 1.03}}$ & $73.00_{\pm 2.20}$ \\
    Actor & $33.23_{\pm 1.16}$ & $39.66_{\pm 0.55}$ & $37.61_{\pm 0.89}$ & $39.92_{\pm 0.67}$ & $\bm{41.79_{\pm1.01}}$ & $41.17_{\pm 0.64}$ & $39.10_{\pm1.59}$ \\
    Squirrel & $46.78_{\pm 0.87}$ & $34.71_{\pm 0.57}$ & $40.55_{\pm 0.42}$ & $50.15_{\pm 1.92}$ & $51.35_{\pm 0.73}$ & $57.38_{\pm 1.25}$ & $\bm{60.61_{\pm 1.69}}$ \\
    Texas & $77.38_{\pm 3.28}$ & $90.98_{\pm 1.64}$ & $86.22_{\pm 2.45}$ & $92.95_{\pm 1.31}$ & $93.12_{\pm 0.65}$ & $\bm{93.44_{\pm 2.13}}$ & $87.54_{\pm 3.45}$ \\
    Cornell & $65.90_{\pm 4.43}$ & $91.81_{\pm 1.96}$ & $83.93_{\pm 2.13}$ & $91.37_{\pm 1.81}$ & $92.13_{\pm 1.64}$ & $\bm{92.95_{\pm 2.46}}$ & $84.47_{\pm 3.45}$ \\
    \bottomrule  
    
    \end{tabular}
    \vskip -0.2in
\end{table*}

\section{Experiment}
This section presents the evaluation of the performance of {\model}.
We first use {\model} to learn the filters of the image datasets, then show the performance of {\model} on real datasets, report the ablation analysis of {\model}, and finally propose the evaluation on time cost.
All experiments are conducted on a machine with $4$ NVIDIA V100 32GB GPUs, Intel Xeon 96-core 2.50GHz CPU, and 400GB RAM. 
\ifsourcecode
Our code is available at \url{https://github.com/TaoLbr1993/LON-GNN}.
\else
\fi

\subsection{Evaluation on Learning Filters}

Following previous work~\cite{nips21bernnet,icml22jacobi}, we use the proposed models to learn the filters from $50$ images. 
Each image is treated as a graph, where every pixel represents a node and two neighboring pixels have an edge between the corresponding nodes.
$5$ typical filters, \ie low ($\exp\{-10\lambda^2\}$), high ($1-\exp\{-10\lambda^2\}$), band ($\exp\{-10(\lambda-1)^2\}$) and rejection ($|\sin{\pi\lambda}|$), are applied to the graph signal and the average square loss of each type of filters is reported.
We compare {\model} with previous models based on Spectral GNNs, including GPRGNN~\cite{iclr21gprgnn}, ARMA~\cite{pami22arma}, ChebNet~\cite{nips16chebyshev}, BernNet~\cite{nips21bernnet}, and JacobiConv~\cite{icml22jacobi}.

The implementation and searching space of {\model} follow a setting in ~\cite{icml22jacobi}, \ie the MLP and all dropout are removed and the coefficients of the polynomial are fixed as the initialization values.
Following ~\cite{icml22jacobi}, we use optuna~\cite{kdd19optuna} to choose the best parameters.
The learning rate and weight decay for parameters except for $a$ and $b$ are chosen from $\{\expnumber{5}{-4}, 0.001, 0.005, 0.01, 0.05\}$ and $\{0.0, \expnumber{5}{-5}, \expnumber{1}{-4}, \expnumber{5}{-4}, 0.001\}$, respectively.
The learning rate and weight decay of $a$ and $b$ is fixed as $0.0002$ and $0.0$ respectively.

\tabref{tab:exp-res-fit} reports the average square loss of the baselines and {\model}.
{\model} obtains the same loss with the best baseline JacobiConv on LOW and REJECT, and outperforms all baselines on HIGH, BAND and COMB.
Comparing with the baselines, {\model} decreases the loss by at least $72.7\%$, $26.8\%$, and $2.2\%$ on HIGH, BAND, and COMB, respectively.
Specifically, {\model} achieves a notable improvement on HIGH and BAND.
This implies that the spectral GNNs with learnable parameters have a more fitting capability compared with the baselines.

\begin{table*}[t!]
    \caption{Accuracy and Improvements of Jacobi+OrthNorm, Jacobi+Lernable and LON-GNN compared with SOTA baseline.}
    \label{tab:exp-ablation}

    \begin{center}
    \begin{tabular}
    {c|c|cc|cc|cc}
    \toprule
      & JacobiConv & \multicolumn{2}{c|}{Jacobi+OrthNorm} & \multicolumn{2}{c|}{Jacobi+Learnable} & \multicolumn{2}{c}{LON-GNN}  \\
      & Accuracy & Accuracy & Improvement & Acuracy & Improvement & Accuracy & Improvement \\
    \midrule
    Cora & $88.98$ & $88.52$ & $\downarrow 0.46$ & $89.28$ & $\uparrow 0.30$ & $89.44$ & $\uparrow 0.46$ \\
    Citeseer & $80.78$ & $80.93$ & $\downarrow 0.15$ & $81.64$ & $\uparrow 0.86$ & $81.41$ & $\uparrow 0.63$ \\
    Pubmed & $89.62$ & $89.93$ & $\uparrow 0.29$ & $89.93$ & $\uparrow 0.29$ & $90.98$ & $\uparrow 1.36$ \\
    Computers & $90.39$ & $90.51$ & $\uparrow 0.12$ & $90.66$ & $\uparrow 0.27$ & $90.93$ & $ \uparrow 0.54$ \\
    Photo & $95.43$ & $94.60$ & $\downarrow 0.83$ & $95.40$ & $\downarrow 0.03$ & $94.66$ & $\downarrow 0.77$ \\
    Chameleon & $74.20$ & $73.70$ & $\downarrow 0.50$ & $74.51$ & $\uparrow 0.31$ & $73.00$ & $\downarrow 1.20$ \\
    Actor & $41.17$ & $41.42$ & $\uparrow 0.25$ & $40.73$ & $\downarrow 0.44$ & $39.10$ & $\downarrow 2.07$ \\
    Squirrel & $57.38$ & $60.11$ & $\uparrow 2.73$ & $57.65$ & $\uparrow 0.27$ & $60.61$ & $\uparrow 3.23$ \\
    Texas & $93.44$ & $84.43$ & $\downarrow 9.01$ & $93.44$ & $\rightarrow 0.00$ & $87.52$ & $\downarrow 6.92$ \\
    Cornell & $92.95$ & $83.83$ & $\downarrow 9.12$ & $91.91$ & $\downarrow 1.04$ & $84.47$ & $\downarrow 7.44$ \\
    \midrule 
    \tabincell{c}{Improved Datasets} & - & \multicolumn{2}{c|}{5} & \multicolumn{2}{c|}{6} & \multicolumn{2}{c}{5} \\
    \bottomrule  
    
    \end{tabular}
    \end{center}
\end{table*}

\subsection{Evaluation on Real-World Datasets}
We also evaluate the performance of the proposed models on both homogeneous and heterogenous real-world datasets, the statistics of which are presented in \tabref{tab:exp-sta}.
The homophily of a graph is defined as the ratio of the number of edges connecting nodes with the same label to the number of all edges.
Depending on the level of homophily, datasets can be categorized into homogeneous and heterogeneous ones~\cite{nips20beyond-homo}.
For homogeneous datasets, we use the citation datasets Cora, Citeseer, and Pubmed~\cite{AImag08collect}, and the Amazon co-purchase datasets Computers and Photo~\cite{corr18pitfalls}.
For heterogeneous datasets, we choose the wikipedia datasets Chameleon and Squirrel~\cite{jcnetwork21multiscale}, the Actor co-occurrence datasets~\cite{iclr20geom}, and the WebKB datasets\footnote{http://www.cs.cmu.edu/afs/cs.cmu.edu/project/theo-11/www/wwkb/}.

Following previous works~\cite{iclr21gprgnn,nips21bernnet,icml22jacobi}, we randomly split the nodes into disjoint train/validation/test sets with ratio 60\%/20\%/20\%.
Expriments on each dataset are repeated 10 times and the average accuracy and $\pm 95\%$ confidence interval on test node sets are reported.
We compare {\model} with GCN~\cite{iclr17gcn}, APPNP~\cite{iclr19appnp}, ChebNet~\cite{nips16chebyshev}, GPRGNN~\cite{iclr21gprgnn}, BernNet~\cite{nips21bernnet}, and JacobiConv~\cite{icml22jacobi}.
In addition to the learning rate and weight decay in vanilla spectral GNNs, the learning rate of $a$ and $b$ is also considered in searching space and chosen from $\{\expnumber{5}{-4}, 0.001, 0.005, 0.01, 0.05\}$.

\tabref{tab:exp-real-world-res} shows the experimental results on real-world datasets.
Among $10$ dataset, {\model} beats all baselines on $5$ datasets, \ie Cora, Citeseer, Pubmed, Computers and Squirrel, and achieves the runner-up performance on Photo and Chameleon, which is competitive to the baselines.
Specifically, {\model} outperforms the baselines by at least $1.36\%$ on the homogeneous dataset Pubmed and $5.63\%$ on the heterogeneous dataset Squirrel.
The experimental results imply the competitiveness of {\model} on both homogeneous and heterogeneous datasets.

\subsection{Ablation Analysis}

To analyze the effectiveness of different techniques in {\model}, we compare {\model} with spectral GNNs with Jacobi polynomial basis (\ie JacobiConv in \cite{icml22jacobi}, with orthonormal Jacobi basis (namely Jacobi+OrthNorm), and with learnable Jacobi basis (namely Jacobi+Learnable).

\tabref{tab:exp-ablation} illustrates the performance of the $4$ models.
On $4$ datasets (Cora, Pubmed, Computers and Squirrel), Jacobi+OrthNorm and Jacobi+Learnable outperform JacobiConv and {\model} performs the best.
By the utilization of orthonormal basis and learnable parameters, all of the three algorithms, \ie Jacobi+OrthNorm, Jacobi+Learnable and LON-GNN, outperform JacobiConv on at least $5$ real-world datasets.
The experimental results implies the effectiveness of the proposed techniques in the paper.

\begin{table*}[t!]
    \caption{Experimental Results on Time Cost: Average Time per Epoch (ms)/Total Training Time (s).}
    \label{tab:exp-time-cost}

    \begin{center}
    \begin{tabular}{ccccccccccc}

    \toprule
     Datasets & GCN & APPNP & ChebNet & GPRGNN & BernNet & JacobiConv & {\model}  \\
    \midrule
    Cora & $6.16/3.67$ & $9.30/3.41$ & $9.16/3.50$ & $10.55/4.35$ & $32.31/12.97$ & $18.72/8.53$ & $37.29/14.09$ \\
    Citeseer & $6.81/1.66$ & $9.79/3.30$ & $12.59/3.12$ & $10.37/2.37$ & $33.01/9.13$ & $19.70/4.25$ & $33.81/12.84$  \\
    Pubmed & $7.03/4.66$ & $9.82/5.45$ & $12.30/7.72$ & $10.34/6.78$ & $33.82/22.30$ & $19.98/18.36$ & $39.18/28.15$ \\
    Computers & $8.73/5.40$ & $11.97/7.39$ & $38.53/24.77$ & $12.55/9.23$ & $55.65/36.84$ & $22.86/19.09$ & $41.87/32.11$ \\
    Photo & $6.80/5.24$ & $9.41/6.90$ & $20.77/16.78$ & $10.03/6.11$ & $34.24/25.55$ & $18.99/16.31$ & $37.32/30.28$ \\
    Chameleon & $6.13/4.43$ & $8.53/4.68$ & $27.81/11.87$ & $9.04/8.45$ & $24.91/15.19$ & $20.24/18.17$ & $42.24/38.09$ \\
    Actor & $7.01/1.62$ & $8.78/3.89$ & $18.05/4.37$ & $9.75/2.91$ & $25.05/8.17$ & $18.79/5.07$ & $39.52/15.67$ \\
    Squirrel & $7.55/3.86$ & $8.35/4.08$ & $87.14/38.14$ & $8.96/7.18$ & $31.21/30.52$ & $18.02/17.79$ & $39.04/36.49$ \\
    Texas & $6.20/3.61$ & $9.47/4.20$ & $7.68/4.37$ & $10.37/2.77$ & $32.89/12.35$ & $18.25/8.58$ & $37.35/17.11$ \\
    Cornell & $5.94/2.84$ & $9.17/2.93$ & $6.94/4.93$ & $9.81/2.11$ & $29.70/14.39$ & $17.65/9.43$ & $37.18/11.44$ \\
    \bottomrule  
    
    \end{tabular}
    \end{center}
\end{table*}

\subsection{Evaluation on Time Cost}
To evaluate the scalability of {\model}, we test the average time cost per epoch and the total training time of the baselines and {\model} on real-world datasets.

The experimental results are reported in \tabref{tab:exp-time-cost}.
{\model} spends more time than the baselines in most cases.
Compared with JacobiConv, {\model} spends up to $2.17\times$ times of time cost per epoch.
This is because JacobiConv utilizes a linear regression to embed the features, while {\model} uses an MLP.
Compared with BernNet, {\model} spends up to $1.70\times$ times of time cost per epoch.
The gap is mainly derived from the extra computation of the norm of orthogonal polynomials and learnable parameters $a$ and $b$.
In consideration of the SOTAs, the time cost of {\model} can be tolerant in practice.

\section{Related Work}
\label{sec:rel}
In recent years, graph neural networks (GNN)~\cite{tnnls20survey-gnn,aiopen2020gnn-survey} have achieved remarkable progress in graph representation learning~\cite{hamilton2020grl}. Its successful applications cover various levels of tasks~\cite{iclr17gcn,arxiv17gat,nips18seal,aaai18graph-class}. In literature, the motivations behind GNN are often summarized from three different perspectives: graph signal processing~\cite{pieee18gsp}, message passing-based inference in probabilistic graphical models~\cite{icml16discrim}, and Weisfeiler-Lehman (WL) graph isomorphism test~\cite{iclr19powerful,aaai19wlgo}.

On the one hand, the first perspective leads to spectral GNNs, which can be roughly categorized into GNNs with fixed filters and GNNs with adaptive filters. Some pioneering studies such as ChebNet~\cite{nips16chebyshev}, GCN~\cite{iclr17gcn}, and SGC~\cite{icml19sgc} adopt a fixed filter. To be capable of fitting arbitrary graph filters, recently proposed spectral GNNs utilize polynomial approximation~\cite{nips21bernnet,nips22evennet,icml22jacobi}. Our proposed LON-GNN is in this line, which leverages Jacobi polynomials.

On the other hand, GNNs motivated from the third perspective are instead regarded as spatial GNNs. The expressiveness of these GNNs is often analyzed by comparing them to the WL-test~\cite{iclr19powerful}, and recent works explore higher-order WL-tests or even more complicated graph algorithms~\cite{pami23subgraph,nips21local}.

There are also other perspectives for interpreting GNNs, such as random walks~\cite{iclr19appnp,iclr18fastgcn}, optimizing an energy function~\cite{icml21iterative}, and geometric deep learning~\cite{icml21breaking,icml19universial}.

The generalization of GNNs can now be analyzed with neural tangent kernel~\cite{nips19gntk}, and applicability to large-scale graphs is empowered with various graph sampling algorithms~\cite{nips17graphsage,iclr20graphsaint}.

Besides, there are some works that introduce the attention mechanisms to GNN to weight neighbors of nodes~\cite{arxiv17gat,corr18attention}, meta-paths~\cite{nips19gtn}, \etc

Compared with these works, we propose {\model} from the perspective of designing a learnable polynomial basis and regularization on the output values of filters.

\section{Conclusion and Future Directions}
\label{sec:con}
In this paper, we first identify the so-called over-passing issue and provide empirical and theoretical evidence. To alleviate this issue, namely,  regularizing the output values of a graph filter, we propose a principled way of regularizing the filter function. Meanwhile, we show that trivially regularizing polynomial coefficients, which most existing spectral GNNs adopt, could be more theoretically sound. All these naturally motivate us to propose {\model}, a spectral GNN that optimizes not only the parameters of vanilla spectral GNNs as in previous works but also the parameters of Jacobi polynomials such that suitable orthonormal basis can be effectively learned and regularized.
Experiments on fitting filters and node classification tasks validate the superiority of {\model}.

An interesting direction for future works is to make more parameters in spectral GNNs learnable, like the highest degree of polynomial $K$, so that more flexible GNNs could be achieved.
c
\nocite{langley00}

\clearpage
\bibliography{longnn}

\begin{thebibliography}{43}
\providecommand{\natexlab}[1]{#1}
\providecommand{\url}[1]{\texttt{#1}}
\expandafter\ifx\csname urlstyle\endcsname\relax
  \providecommand{\doi}[1]{doi: #1}\else
  \providecommand{\doi}{doi: \begingroup \urlstyle{rm}\Url}\fi

\bibitem[Akiba et~al.(2019)Akiba, Sano, Yanase, Ohta, and Koyama]{kdd19optuna}
Akiba, T., Sano, S., Yanase, T., Ohta, T., and Koyama, M.
\newblock Optuna: {A} next-generation hyperparameter optimization framework.
\newblock In \emph{{SIGKDD}}, pp.\  2623--2631, 2019.

\bibitem[Balcilar et~al.(2021)Balcilar, H{\'{e}}roux, Ga{\"{u}}z{\`{e}}re,
  Vasseur, Adam, and Honeine]{icml21breaking}
Balcilar, M., H{\'{e}}roux, P., Ga{\"{u}}z{\`{e}}re, B., Vasseur, P., Adam, S.,
  and Honeine, P.
\newblock Breaking the limits of message passing graph neural networks.
\newblock In \emph{{ICML}}, pp.\  599--608, 2021.

\bibitem[Barcel{\'{o}} et~al.(2021)Barcel{\'{o}}, Geerts, Reutter, and
  Ryschkov]{nips21local}
Barcel{\'{o}}, P., Geerts, F., Reutter, J.~L., and Ryschkov, M.
\newblock Graph neural networks with local graph parameters.
\newblock In \emph{NeurIPS}, pp.\  25280--25293, 2021.

\bibitem[Bianchi et~al.(2022)Bianchi, Grattarola, Livi, and Alippi]{pami22arma}
Bianchi, F.~M., Grattarola, D., Livi, L., and Alippi, C.
\newblock Graph neural networks with convolutional {ARMA} filters.
\newblock \emph{{IEEE} Transactions on Pattern Analysis and Machine
  Intelligence}, 44\penalty0 (7):\penalty0 3496--3507, 2022.

\bibitem[Bouritsas et~al.(2023)Bouritsas, Frasca, Zafeiriou, and
  Bronstein]{pami23subgraph}
Bouritsas, G., Frasca, F., Zafeiriou, S., and Bronstein, M.~M.
\newblock Improving graph neural network expressivity via subgraph isomorphism
  counting.
\newblock \emph{{IEEE} Transactions on Pattern Analysis and Machine
  Intelligence}, 45\penalty0 (1):\penalty0 657--668, 2023.

\bibitem[Chen et~al.(2018)Chen, Ma, and Xiao]{iclr18fastgcn}
Chen, J., Ma, T., and Xiao, C.
\newblock Fastgcn: Fast learning with graph convolutional networks via
  importance sampling.
\newblock In \emph{{ICLR}}, 2018.

\bibitem[Chen et~al.(2020)Chen, Chen, Zhang, Ji, Fu, Zhao, Chen, and
  Lu]{corr20briginggap}
Chen, Z., Chen, F., Zhang, L., Ji, T., Fu, K., Zhao, L., Chen, F., and Lu, C.
\newblock Bridging the gap between spatial and spectral domains: {A} survey on
  graph neural networks.
\newblock \emph{CoRR}, abs/2002.11867, 2020.

\bibitem[Chien et~al.(2021)Chien, Peng, Li, and Milenkovic]{iclr21gprgnn}
Chien, E., Peng, J., Li, P., and Milenkovic, O.
\newblock Adaptive universal generalized pagerank graph neural network.
\newblock In \emph{ICLR}, 2021.

\bibitem[Chihara(2011)]{book2011orth}
Chihara, T.~S.
\newblock \emph{An introduction to orthogonal polynomials}.
\newblock Courier Corporation, 2011.

\bibitem[Dai et~al.(2016)Dai, Dai, and Song]{icml16discrim}
Dai, H., Dai, B., and Song, L.
\newblock Discriminative embeddings of latent variable models for structured
  data.
\newblock In \emph{{ICML}}, pp.\  2702--2711, 2016.

\bibitem[Defferrard et~al.(2016)Defferrard, Bresson, and
  Vandergheynst]{nips16chebyshev}
Defferrard, M., Bresson, X., and Vandergheynst, P.
\newblock Convolutional neural networks on graphs with fast localized spectral
  filtering.
\newblock In \emph{NeurIPS}, pp.\  3837--3845, 2016.

\bibitem[Du et~al.(2019)Du, Hou, Salakhutdinov, P{\'{o}}czos, Wang, and
  Xu]{nips19gntk}
Du, S.~S., Hou, K., Salakhutdinov, R., P{\'{o}}czos, B., Wang, R., and Xu, K.
\newblock Graph neural tangent kernel: Fusing graph neural networks with graph
  kernels.
\newblock In \emph{{NeurIPS}}, pp.\  5724--5734, 2019.

\bibitem[Hamilton(2020)]{hamilton2020grl}
Hamilton, W.~L.
\newblock Graph representation learning.
\newblock \emph{Synthesis Lectures on Artifical Intelligence and Machine
  Learning}, 14\penalty0 (3):\penalty0 1--159, 2020.

\bibitem[Hamilton et~al.(2017)Hamilton, Ying, and Leskovec]{nips17graphsage}
Hamilton, W.~L., Ying, Z., and Leskovec, J.
\newblock Inductive representation learning on large graphs.
\newblock In \emph{{NeurIPS}}, pp.\  1024--1034, 2017.

\bibitem[He et~al.(2021)He, Wei, Huang, and Xu]{nips21bernnet}
He, M., Wei, Z., Huang, Z., and Xu, H.
\newblock Bernnet: Learning arbitrary graph spectral filters via bernstein
  approximation.
\newblock In \emph{NeurIPS}, pp.\  14239--14251, 2021.

\bibitem[He et~al.(2022)He, Wei, and Wen]{icml22chebyrevisit}
He, M., Wei, Z., and Wen, J.
\newblock Convolutional neural networks on graphs with chebyshev approximation,
  revisited.
\newblock \emph{CoRR}, abs/2202.03580, 2022.

\bibitem[Kipf \& Welling(2017)Kipf and Welling]{iclr17gcn}
Kipf, T.~N. and Welling, M.
\newblock Semi-supervised classification with graph convolutional networks.
\newblock In \emph{{ICLR}}, 2017.

\bibitem[Klicpera et~al.(2019)Klicpera, Bojchevski, and
  G{\"{u}}nnemann]{iclr19appnp}
Klicpera, J., Bojchevski, A., and G{\"{u}}nnemann, S.
\newblock Predict then propagate: Graph neural networks meet personalized
  pagerank.
\newblock In \emph{{ICLR}}, 2019.

\bibitem[Lakshminarayanan \& Fleck(2011)Lakshminarayanan and
  Fleck]{lakshminarayanan2011zernike}
Lakshminarayanan, V. and Fleck, A.
\newblock Zernike polynomials: a guide.
\newblock \emph{Journal of Modern Optics}, 58\penalty0 (7):\penalty0 545--561,
  2011.

\bibitem[Lei et~al.(2022)Lei, Wang, Li, Ding, and Wei]{nips22evennet}
Lei, R., Wang, Z., Li, Y., Ding, B., and Wei, Z.
\newblock Evennet: Ignoring odd-hop neighbors improves robustness of graph
  neural networks.
\newblock In \emph{NeurIPS}, pp.\  1--13, 2022.

\bibitem[Maron et~al.(2019)Maron, Fetaya, Segol, and Lipman]{icml19universial}
Maron, H., Fetaya, E., Segol, N., and Lipman, Y.
\newblock On the universality of invariant networks.
\newblock In \emph{{ICML}}, pp.\  4363--4371, 2019.

\bibitem[Morris et~al.(2019)Morris, Ritzert, Fey, Hamilton, Lenssen, Rattan,
  and Grohe]{aaai19wlgo}
Morris, C., Ritzert, M., Fey, M., Hamilton, W.~L., Lenssen, J.~E., Rattan, G.,
  and Grohe, M.
\newblock Weisfeiler and leman go neural: Higher-order graph neural networks.
\newblock In \emph{{AAAI}}, pp.\  4602--4609, 2019.

\bibitem[Ortega et~al.(2018)Ortega, Frossard, Kovacevic, Moura, and
  Vandergheynst]{pieee18gsp}
Ortega, A., Frossard, P., Kovacevic, J., Moura, J. M.~F., and Vandergheynst, P.
\newblock Graph signal processing: Overview, challenges, and applications.
\newblock \emph{Proceedings of the {IEEE}}, 106\penalty0 (5):\penalty0
  808--828, 2018.

\bibitem[Pei et~al.()Pei, Wei, Chang, Lei, and Yang]{iclr20geom}
Pei, H., Wei, B., Chang, K.~C., Lei, Y., and Yang, B.
\newblock Geom-gcn: Geometric graph convolutional networks.
\newblock In \emph{{ICLR}}.

\bibitem[Rozemberczki et~al.(2021)Rozemberczki, Allen, and
  Sarkar]{jcnetwork21multiscale}
Rozemberczki, B., Allen, C., and Sarkar, R.
\newblock Multi-scale attributed node embedding.
\newblock \emph{Journal of Complex Networks}, 9\penalty0 (2), 2021.

\bibitem[Sen et~al.(2008)Sen, Namata, Bilgic, Getoor, Gallagher, and
  Eliassi{-}Rad]{AImag08collect}
Sen, P., Namata, G., Bilgic, M., Getoor, L., Gallagher, B., and Eliassi{-}Rad,
  T.
\newblock Collective classification in network data.
\newblock \emph{{AI} Magazine}, 29\penalty0 (3):\penalty0 93--106, 2008.

\bibitem[Shchur et~al.(2018)Shchur, Mumme, Bojchevski, and
  G{\"{u}}nnemann]{corr18pitfalls}
Shchur, O., Mumme, M., Bojchevski, A., and G{\"{u}}nnemann, S.
\newblock Pitfalls of graph neural network evaluation.
\newblock \emph{CoRR}, abs/1811.05868, 2018.

\bibitem[Smith(1997)]{smith1997guidetodsp}
Smith, S.~W.
\newblock The scientist and engineer's guide to digital signal processing,
  1997.

\bibitem[Thekumparampil et~al.(2018)Thekumparampil, Wang, Oh, and
  Li]{corr18attention}
Thekumparampil, K.~K., Wang, C., Oh, S., and Li, L.
\newblock Attention-based graph neural network for semi-supervised learning.
\newblock \emph{CoRR}, abs/1803.03735, 2018.

\bibitem[Van Der~Heijden et~al.(2005)Van Der~Heijden, Duin, De~Ridder, and
  Tax]{van2005classification}
Van Der~Heijden, F., Duin, R.~P., De~Ridder, D., and Tax, D.~M.
\newblock \emph{Classification, parameter estimation and state estimation: an
  engineering approach using MATLAB}.
\newblock John Wiley \& Sons, 2005.

\bibitem[Veli{\v{c}}kovi{\'c} et~al.(2017)Veli{\v{c}}kovi{\'c}, Cucurull,
  Casanova, Romero, Lio, and Bengio]{arxiv17gat}
Veli{\v{c}}kovi{\'c}, P., Cucurull, G., Casanova, A., Romero, A., Lio, P., and
  Bengio, Y.
\newblock Graph attention networks.
\newblock \emph{arXiv}, 2017.

\bibitem[Wang \& Zhang(2022)Wang and Zhang]{icml22jacobi}
Wang, X. and Zhang, M.
\newblock How powerful are spectral graph neural networks.
\newblock In \emph{ICML}, pp.\  23341--23362, 2022.

\bibitem[Wu et~al.(2019)Wu, Jr., Zhang, Fifty, Yu, and Weinberger]{icml19sgc}
Wu, F., Jr., A. H.~S., Zhang, T., Fifty, C., Yu, T., and Weinberger, K.~Q.
\newblock Simplifying graph convolutional networks.
\newblock In \emph{{ICML}}, pp.\  6861--6871, 2019.

\bibitem[Wu et~al.(2020)Wu, Pan, Chen, Long, Zhang, and
  Philip]{tnnls20survey-gnn}
Wu, Z., Pan, S., Chen, F., Long, G., Zhang, C., and Philip, S.~Y.
\newblock A comprehensive survey on graph neural networks.
\newblock \emph{IEEE transactions on Neural Networks and Learning Systems},
  32\penalty0 (1):\penalty0 4--24, 2020.

\bibitem[Xu et~al.(2019)Xu, Hu, Leskovec, and Jegelka]{iclr19powerful}
Xu, K., Hu, W., Leskovec, J., and Jegelka, S.
\newblock How powerful are graph neural networks?
\newblock In \emph{{ICLR}}, 2019.

\bibitem[Xu et~al.(2021)Xu, Zhang, Jegelka, and Kawaguchi]{icml21optim-gnn}
Xu, K., Zhang, M., Jegelka, S., and Kawaguchi, K.
\newblock Optimization of graph neural networks: Implicit acceleration by skip
  connections and more depth.
\newblock In \emph{{ICML}}, pp.\  11592--11602, 2021.

\bibitem[Yang et~al.(2021)Yang, Liu, Wang, Zhou, Gan, Wei, Zhang, Huang, and
  Wipf]{icml21iterative}
Yang, Y., Liu, T., Wang, Y., Zhou, J., Gan, Q., Wei, Z., Zhang, Z., Huang, Z.,
  and Wipf, D.
\newblock Graph neural networks inspired by classical iterative algorithms.
\newblock In \emph{{ICML}}, pp.\  11773--11783, 2021.

\bibitem[Yun et~al.(2019)Yun, Jeong, Kim, Kang, and Kim]{nips19gtn}
Yun, S., Jeong, M., Kim, R., Kang, J., and Kim, H.~J.
\newblock Graph transformer networks.
\newblock In \emph{{NeurIPS}}, pp.\  11960--11970, 2019.

\bibitem[Zeng et~al.(2020)Zeng, Zhou, Srivastava, Kannan, and
  Prasanna]{iclr20graphsaint}
Zeng, H., Zhou, H., Srivastava, A., Kannan, R., and Prasanna, V.~K.
\newblock Graphsaint: Graph sampling based inductive learning method.
\newblock In \emph{{ICLR}}, 2020.

\bibitem[Zhang \& Chen(2018)Zhang and Chen]{nips18seal}
Zhang, M. and Chen, Y.
\newblock Link prediction based on graph neural networks.
\newblock In \emph{NeurIPS}, pp.\  5171--5181, 2018.

\bibitem[Zhang et~al.(2018)Zhang, Cui, Neumann, and Chen]{aaai18graph-class}
Zhang, M., Cui, Z., Neumann, M., and Chen, Y.
\newblock An end-to-end deep learning architecture for graph classification.
\newblock In \emph{{AAAI}}, pp.\  4438--4445, 2018.

\bibitem[Zhou et~al.(2020)Zhou, Cui, Hu, Zhang, Yang, Liu, Wang, Li, and
  Sun]{aiopen2020gnn-survey}
Zhou, J., Cui, G., Hu, S., Zhang, Z., Yang, C., Liu, Z., Wang, L., Li, C., and
  Sun, M.
\newblock Graph neural networks: A review of methods and applications.
\newblock \emph{AI open}, 1:\penalty0 57--81, 2020.

\bibitem[Zhu et~al.(2020)Zhu, Yan, Zhao, Heimann, Akoglu, and
  Koutra]{nips20beyond-homo}
Zhu, J., Yan, Y., Zhao, L., Heimann, M., Akoglu, L., and Koutra, D.
\newblock Beyond homophily in graph neural networks: Current limitations and
  effective designs.
\newblock In \emph{NeurIPS}, 2020.

\end{thebibliography}
\bibliographystyle{icml2023}

\newpage
\appendix
\onecolumn

\section{Details of Experiments}
\label{sec:app-details-exp}

\stitle{Baselines.} 
We directly report the results of JacobiConv following ~\cite{icml22jacobi}, and the results of other baselines are from ~\cite{nips21bernnet}.

\stitle{Hyperparameter Search.}
For the experiments on real-world datasets and fitting filters, we use optuna~\cite{kdd19optuna} to search the best hyperparameters for spectral GNNs with orthonormal basis.
We choose the best parameters with minimum average square loss for fitting filter experiments and maximum average accuracy on validation datasets for real-world datasets.
Following ~\cite{icml22jacobi}, the learning rate and weight decay for parameters are chosen from $\{\expnumber{5}{-4}, 0.001, 0.005, 0.01, 0.05\}$ and $\{0.0, \expnumber{5}{-5}, \expnumber{1}{-4}, \expnumber{5}{-04}, 0.001\}$, respectively.
We choose $a$ and $b$ in Jacobi polynomials from $[-1,2]$.

For fitting filter experiments, we simply fix the best parameters of other parameters and search the learning rate of $a$ and $b$ from $\{0.01, 0.001, \expnumber{5}{-4}, \expnumber{2}{-4}, \expnumber{1}{-4} \}$ in fitting filter experiments.
For real-world datasets, we also search the learning rate of $a$ and $b$ in optuna, whose values are also chosen from $\{\expnumber{5}{-4}, 0.001, 0.005, 0.01, 0.05\}$. 
The weight decay of $a$ and $b$ is fixed to $0.0$.

\begin{table}[t!]
    \caption{Average Loss of {\model} with Different Learning Rate in Learning Filters Experiments.}
    \label{tab:full-exp-fit}

    \begin{center}
    \begin{tabular}{cccccc}

    \toprule
     & LOW & HIGH & BAND & REJECT & COMB  \\
    \midrule
    0.01    & $0.0003$ & $0.0008$ & $0.0172$ & $0.0156$ & $0.2868$ \\
    0.001    & $0.0003$ & $0.0004$ & $0.0158$ & $0.0156$ & $0.2869$ \\
    0.0005 & $0.0003$ & $0.0003$ & $0.0157$ & $0.0156$ & $0.2869$ \\
    0.0002 & $0.0003$ & $0.0003$ & $0.0156$ & $0.0156$ & $0.2870$ \\
    0.0001 & $0.0003$ & $0.0003$ & $0.0156$ & $0.0156$ & $0.2870$ \\
    \bottomrule

    \end{tabular}
    \end{center}
    \vskip -0.1in
    \end{table}

\stitle{Results of Fitting Filters.}
The full experimental results of Fitting Filters are reported in \tabref{tab:full-exp-fit}.

\end{document}